\documentclass[journal]{IEEEtran}
\usepackage{amsmath,amssymb}
\usepackage{amsfonts}
\usepackage{mathrsfs}
\usepackage{color}
\usepackage{graphicx}
\usepackage{amsmath}
\usepackage{subcaption}
\usepackage{mwe}
\usepackage{hyperref}
\usepackage{siunitx}  

\usepackage{array,booktabs} 
\usepackage[algo2e]{algorithm2e} 
\usepackage{algorithm}
\usepackage{algpseudocode}
\usepackage{amsmath,amssymb} 
\usepackage{stfloats} 
\usepackage{psfrag} 
\usepackage{color}

\definecolor{box_color}{rgb}{.8,.8,.8}

\usepackage[utf8]{inputenc}
\usepackage[T1]{fontenc}

\newtheorem{theorem}{Theorem}
\newtheorem{lemma}{Lemma}
\newtheorem{proposition}{Proposition}
\newtheorem{corollary}{Corollary}
\newtheorem{fact}{Fact}
\newtheorem{definition}{Definition}
\newtheorem{remark}{Remark}

\newtheorem{assumption}{Assumption}
\newtheorem{ass}{C.}

\newenvironment{proof}[1][Proof]{\begin{trivlist}
\item[\hskip \labelsep {\bfseries #1}]}{\end{trivlist}}

\def\begcen{\begin{center}}
\def\endcen{\end{center}}

\newcommand{\bfq}{\mbox{$q$}}

\def\bfq{{\bf q}}

\newcommand{\RE}{\mathbb {R}}    

\newcommand{\col}{ \mbox{col} }
\newcommand{\rank}{ \mbox{rank } }

\def\L2{{\cal L}_2}
\def\L2e{{\cal L}_{2e}}

\def\rea{\mathbb{R}}

\def\diag{\mbox{diag}}

\def\bfq{{\bf q}}

\def\begequarr{\begin{eqnarray}}
\def\endequarr{\end{eqnarray}}
\def\begequarrs{\begin{eqnarray*}}
\def\endequarrs{\end{eqnarray*}}
\def\begarr{\begin{array}}
\def\endarr{\end{array}}
\def\begequ{\begin{equation}}
\def\endequ{\end{equation}}
\def\lab{\label}
\def\begdes{\begin{description}}
\def\enddes{\end{description}}
\def\begenu{\begin{enumerate}}
\def\begite{\begin{itemize}}
\def\endite{\end{itemize}}
\def\endenu{\end{enumerate}}

\def\lef[{\left[\begin{array}}
\def\rig]{\end{array}\right]}
\def\qed{\hfill$\Box \Box \Box$}
\def\begcen{\begin{center}}
\def\endcen{\end{center}}
\def\begrem{\begin{remark}\rm}
\def\endrem{\end{remark}}
\def\begassum{\begin{assumption}}
\def\endassum{\end{assumption}}
\def\begassums{\begin{assumption*}}
\def\endassums{\end{assumption*}}
\def\begassu{\begin{ass}}
\def\endassu{\end{ass}}
\def\beglem{\begin{lemma}}
\def\endlem{\end{lemma}}
\def\begcor{\begin{corollary}}
\def\endcor{\end{corollary}}
\def\begfac{\begin{fact}}
\def\endfac{\end{fact}}

\begin{document}

\title{Automatic Centralized Control of Underactuated Large-scale Multi-robot Systems using a Generalized Coordinate Transformation }
\author{Babak~Salamat, Christopher Johannes Starck and Heiko Hamann
\thanks{Babak Salamat is with University of Luebeck, Institute of Computer Engineering, Luebeck, Germany, e-mail: salamat@iti.uni-luebeck.de.}
\thanks{Heiko Hamann is with University of Luebeck, Institute of Computer Engineering, Luebeck, Germany, e-mail: hamann@iti.uni-luebeck.de.}
\thanks{Christopher Johannes Starck is with University of Luebeck, Luebeck, Germany, e-mail: christopher.starck@student.uni-luebeck.de.}}
\markboth{SALAMAT B., Starck CJ., AND Hamann H.: SUBMITTED TO IEEE Robotics and Automation Letters }%
{Shell \MakeLowercase{\textit{et al.}}: Bare Demo of IEEEtran.cls for IEEE Journals}

\date{}
\maketitle
\begin{abstract}
Controlling large-scale particle or robot systems is challenging because of their high dimensionality. We use a centralized stochastic approach that allows for optimal control at the cost of a central element instead of a decentralized approach.  Previous works are often restricted to the assumption of fully actuated robots. Here we propose an approach for underactuated robots that allows for energy-efficient control of the robot system. We consider a simple task of gathering the robots (minimizing positional variance) and steering them towards a goal point within a bounded area without obstacles. We make two main contributions. First, we present a generalized coordinate transformation for underactuated robots, whose physical properties should be considered. We choose Euler-Lagrange systems that describe a large class of robot systems. Second, we propose an optimal control mechanism with the prime objective of energy efficiency.
We show the feasibility of our approach in numerical simulations and robot simulations.

\end{abstract}
%
\section{Introduction}
\lab{sec1}

The control of large-scale multi-robot or particle systems is challenging due to the high degree of freedom in such distributed systems of loosely coupled mobile robots~\cite{olfatiSaber2007}. 
The published approaches on this subject can roughly be separated in two complementary classes: (A)~centralized approaches assuming complete information and focusing on precision and efficiency~\cite{elamvazhuthi2019mean} and (B)~decentralized approaches assuming only partial observability and focusing on simple reactive and behavior-based control~\cite{hamann2018}.
While both concepts are generally justified, the centralized approach may be almost unavoidable for certain tasks. Here, we investigate a control problem in robot swarms with minimal hardware~\cite{becker13,Shahrokhi2015}.
In the case of a simple robot, such as the Kilobot robot with its minimal equipment of sensors~\cite{rubenstein12}, certain tasks may be infeasible relying on a decentralized approach. The advantage of having simple hardware is, in turn, that possibly many robots can be built to form a large-scale system with high redundancy. The control problem can be thought of as macroscopic or stochastic control of a `cloud of robots' determined by a distribution~\cite{becker13,Shahrokhi2018}.
The controller's input can be, for example, the mean position of all robots and their variance. The output is a global control effort that is 
broadcasted to all
robots or that operates as a force on each robot.
The variance may be calculated based on robot positions~\cite{Shahrokhi2015}, which could be relaxed in a different approach. 
An option is to exploit the environment by gathering robots at flat obstacles until  minimum variance is achieved~\cite{Graham90}.
The control 
iterates over measuring robot positions 
followed by possibly longer periods of not measuring again but relying on the dynamical model of each robot plus adding Brownian noise on positions, velocities, and accelerations. 
This is generally related to mean-field models of multi-robot systems~\cite{elamvazhuthi2019mean} and specifically the concept of assuming microscopically Brownian particles and the resulting macroscopic evolution of a swarm described by a distribution relates directly to known modeling approaches in swarm robotics based on Langevin equations and Fokker-Planck equations~\cite{hamann08b,prorok2011multi}.


We propose an optimal energy-efficient control mechanism that minimizes positional variance and steers the robot system's mean position to a target position. In particular, our work starts by showing that it is possible to obtain mathematically a mapping such that underactuated robot systems take a partial form. However, due to the complexity of the dynamics (coupling of the inertia matrix), it is not possible to design a controller. Another challenge is the fact that the control input matrix $G(q)$ is time-variant. However, in \cite{Ortega2002}, \cite{Nuno2013} and \cite{Acosta2005}, the authors assumed the input matrix to be in the form of $G = [I_m \hspace{0.2 em} 0_s]^\top$. This assumption on the input matrix $G$ can be applied only to simple robot structures. In this paper, we cover the case where $G(q) = [G_u(q) \hspace{0.2 em} G_a(q)]^\top$ has a general form. Therefore, we relax this assumption on the input matrix $G$ differently from what is done in \cite{Ortega2002}, \cite{Nuno2013} and \cite{Acosta2005}. Indeed, finding a transformation to have the robot systems taking a partial form is not straightforward. Nonetheless, a new generalized coordinate transformation framework is proposed to decouple the system. This allows the development of an optimal control mechanism with the prime objective of energy efficiency. In control theory, several techniques exist to design energy-efficient control laws~ \cite{kirk1970optimal}. 
However, the state-dependent Riccati equation (SDRE)~\cite{Cloutier1997} does not cancel nonlinear terms, which is advantageous because canceling such nonlinearities would significantly increase the control signals~\cite{Freeman95}. Furthermore, SDRE parameters and characterizes the system to a state-dependent coefficient (SDC) form that is useful for immediate stability analysis. Then, we show that our control design provides set point tracking (stabilization) with semi-global properties. Our proof is based on the Lyapunov stability criterion~\cite{isidori1995nonlinear}.

We assume that tracking all robots is possible and calculating the variance is feasible. 
In future extensions, we could ensure scalability of this centralized approach by only tracking boundaries of the `robot cloud,' assume a uniform distribution of robots within or even measure robot densities, and hence getting rid of all microscopic details.

%

%
\noindent {\bf Notation.} $I_n$ is the $n \times n$ identity matrix and $0_{n \times s}$ is an
$n \times s$ matrix of zeros, and $0_n$ is an $n$--dimensional column vector of zeros. Given $a_i \in \rea,\; i \in \bar n := \{1,\dots,n\}$, we denote with $\col(a_i)$ the $n$--dimensional column vector with
elements~$a_i$. For any matrix $A \in \rea^{n \times n}$, $(A)_{i} \in \rea^n$ denotes the $i$--th column, $(A)^{i}$ the $i$--th row and $(A)_{ij}$ the $ij$--th element. 
With $e_i\in \rea^n,\; i \in
\bar n$ we denote the Euclidean basis vectors, $(A)_i := A e_i$, $(A)^{i}:=e_i^\top A$ and $(A)_{ij}:= e_i^{\top} A e_j$. For $x \in \rea^n$, $S \in \rea^{n \times n}$, $S=S^\top
>0$, we denote the Euclidean norm $|x|^2:=x^\top x$, and the weighted--norm $\|x\|^2_S:=x^\top S x$. Given a function $f:  \rea^n \to \rea$ we define the differential operators
$$
\nabla f:=\left(\frac{\displaystyle \partial f }{\displaystyle \partial x}\right)^\top,\;\nabla_{x_i} f:=\left(\frac{\displaystyle
\partial f }{\displaystyle \partial x_i}\right)^\top,
$$
where $x_i \in \rea^p$ is an element of the vector $x$. For a mapping $g : \rea^n \to \rea^m$, its Jacobian matrix is defined as
$$\nabla g:=\left [\begin{array}{cc}(\nabla g_1)^\top \\
\vdots\\ (\nabla g_m)^\top \end{array}\right],$$ where $g_i:\rea^n \to \rea$ is the $i$-th element of $g$.
%
\section{Euler-Lagrange dynamics}
\lab{sec2}
%
We consider an underactuated robot system with dynamics described by the well-known Euler-Lagrange (EL) equations of motion (second-order ordinary differential equations)
 \begequarr 
 \label{lagr}
M(q) \ddot{q} + C(q,\dot{q}) \dot{q} + \nabla V(q) = G(q)  u,
 \endequarr
where $q \in \rea^n$ are the configuration variables, $u \in \RE^m$ are the control signals, $M(q)>0$ is the generalized inertia matrix, $C(q,\dot{q})$ represent the Coriolis and centrifugal forces, $V(q)$ is the systems potential energy, and $G(q)$ is the input matrix.

First, we make an assumption characterizing the class of generalized coordinate transformation~$T$ that we use here.
\begite
\item[{\bf A1.}] There exists an invertible mapping $\Phi:\rea^n \to \rea^n$, such that
\begin{equation}
\nabla_q \Phi(q) = T^{-1}(q).
\end{equation}
is invertible for all $q$.
\endite

\beglem \label{lem1}
Consider a mapping $\Phi: \rea^n \to \rea^n$ that satisfies {\bf A.1} and define the generalised coordinate transformation as follows
\begin{equation} \label{cot}
\bfq= \Phi(q).
\end{equation}
Then, the EL dynamics \eqref{lagr} can be written as follows
\begin{equation} \label{lagrcot}
\mathcal{M}(\bfq) \ddot{\bfq} + \mathcal C(\bfq,\dot{\bfq}) \dot{\bfq} + \nabla \mathcal V(\bfq) = \mathcal G(\bfq)  u,  
\end{equation}
where
\begin{eqnarray}
\dot \bfq &:=& T^{-1}(q) \dot q \label{dcot}\\
\mathcal{M}(\bfq) &:=& T^\top(q) M(q) T(q) \Big|_{q=\Phi^{-1}(\bfq)} \\
\mathcal V(\bfq) &:=& V(q)\Big|_{q=\Phi^{-1}(\bfq)} \\
\mathcal G &:=& T^\top(q) G(q) \Big|_{q=\Phi^{-1}(\bfq)}
\end{eqnarray}
and $\mathcal C(\bfq,\dot{\bfq})\dot{\bfq}$ are the Coriolis and centrifugal forces associated with mass matrix~$\mathcal M(\bfq)$ that we can compute by
\begin{equation}
\mathcal C(\bfq,\dot{\bfq})\dot{\bfq} = \left[ \nabla_\bfq [\mathcal M(\bfq) \dot \bfq] - \frac 12 \nabla_\bfq^\top [\mathcal M(\bfq) \dot \bfq] \right] \dot \bfq.
\end{equation}
The Lagrangian in the new generalised coordinates is
\begin{equation}
\mathcal L (\bfq,\dot \bfq) = \frac 12 \dot \bfq^\top \mathcal{M}(\bfq) \dot \bfq  -  \mathcal V(\bfq).  
\end{equation}
\endlem
\begin{proof}
The proof follows from the coordinate invariance property of the EL equations (or from straightforward calculation computing the derivative of the coordinate transformation and using the original dynamics).
\end{proof}
\begin{remark} Notice that the matrix $T(.)$ can be used to shape the form of the mass matrix $\mathcal M(.)$ in the new generalized coordinates. However, 
we consider only
invertible matrices $T(.)$ that satisfy the integrability assumption {\bf A.1}. That is, given an invertible matrix $T(.)$, we assume there exists an invertible mapping $\Phi: \rea^n \to \rea^n$ that satisfies
$$
\dot \Phi(q) = T(q) \dot{q}.
$$
Therefore, the generalized coordinated transformation \eqref{cot} is well-defined.
\end{remark}
We consider now mechanical systems \eqref{lagr} with an input matrix of the general form
\begequarr
G(q) = \lef[{c} G_{u}(q)\\ G_{a}(q)  \rig], \label{G general}
\endequarr
where $\rank G(q) = m<n $, and $G_{a}(q)$ is an invertible $m \times m$ matrix. $G_u(q)$ and $G_a(q)$ are the underactuated and actuated components of $G(q)$ , respectively.
The EL dynamics (\ref{lagr}) is coupled when $G_{u}(q) \not \equiv 0$.


Furthermore, to simplify the notation, we partition the generalized coordinates and velocity  as $q=\col(q_u,q_a)$, ${\dot q}=  \col ({\dot q}_u, {\dot q}_a)$ with $q_a,\dot q_a \in \RE^m$ and $q_u,\dot q_u \in \RE^{s}$, and partition the inertia and Coriolis matrices as
 \begequarrs
 M (q)= \left[ \begarr{cc} m_{uu} (q) & m^\top_{au} (q) \\ m_{au} (q) & m_{aa}  (q) \endarr \right],
 \endequarrs
\begequarrs
 C (q,\dot q)= \left[ \begarr{cc} c_{uu} (q) & c_{ua} (q) \\ c_{au} (q) & c_{aa}  (q) \endarr \right],
 \endequarrs
where $m_{aa}:\rea^n \to \RE^{m\times m}$, $m_{au}:\rea^n \to \RE^{s \times m}$, $m_{uu}:\rea^n \to \RE^{s\times s}$, $c_{aa}:\rea^n \times \rea^n \to \rea^{m\times m}$,  $c_{au}:\rea^n \times \rea^n \to \rea^{s \times m}$, $c_{ua}:\rea^n \times \rea^n \to \rea^{m \times s}$,  $c_{uu}:\rea^n \to \RE^{s\times s}$.

Next, we impose several assumptions to show particular forms of the EL dynamics \eqref{lagr} under generalized coordinate transformations. 

\begite
\item[{\bf A2.}] There exists a function $\Phi_a:\rea^m \to \rea^s$, such that
\begin{equation}
\dot \Phi_a(q_a)=m_{uu}^{-1} m_{au}^{\top} \dot q_a.
\end{equation}

\item[{\bf A3.}] The inertia matrix depends only on the actuated variables $q_a$, {\em i.e.},  $M(q)=M(q_a)$.

\item[{\bf A4.}] The sub-block matrix $m_{uu}$ of the inertia matrix is constant.

\item[{\bf A5.}] The potential energy can be written as
$$
V(q)=V_a(q_a)+V_u(q_u).
$$
\endite

\begin{proposition} \label{prop1} The  dynamics of the system \eqref{lagr}, under assumption {\bf A.2} and using the generalised coordinates $\bfq=\text{\em col}(\bfq_1,\bfq_2)= \Phi(q)$, can be written as follows
\begequarr
&& {\bf m}_{uu} \ddot{\bfq}_1 + \left[ \nabla_{\bfq_1} ({\bf m}_{uu} \dot \bfq_1) - \frac 12 \nabla^\top_{\bfq_1} ({\bf m}_{uu}^s \dot \bfq_1)\right] \dot \bfq_1 +\nonumber \\ 
&& \left[ \nabla_{\bfq_2} ({\bf m}_{uu}^s \dot \bfq_2) - \frac 12 \nabla^\top_{\bfq_1} ({\bf m}_{aa} \dot \bfq_2)\right] \dot \bfq_2 + \nonumber \\ 
&& \nabla_{\bfq_1} \mathcal V(\bfq) = G_{u}(q) u \label{ddq1_1} \\
&& {\bf m}_{aa}^s \ddot{\bfq}_2 + \left[ \nabla_{\bfq_1} ({\bf m}_{aa}^s \dot \bfq_2) - \frac 12 \nabla^\top_{\bfq_2} ({\bf m}_{uu} \dot \bfq_1)\right] \dot \bfq_1 \nonumber \\
&&+  \left[ \nabla_{\bfq_2} ({\bf m}_{aa}^s \dot \bfq_2) - \frac 12 \nabla^\top_{\bfq_2} ({\bf m}_{aa}^s \dot \bfq_2)\right] \dot \bfq_2 + \nonumber \\ 
&& \nabla_{\bfq_2} \mathcal V(\bfq) =  \bigg[ G_{a}(q)-G_{u}(q)m_{au}m^{-1}_{uu} \bigg] u, \label{ddq2_1}
\endequarr
\end{proposition}
where
\begequarr
\left[\begin{array}{c} \bfq_1 \\ \bfq_2 \end{array} \right] &=& \left[\begin{array}{c} q_u + \Phi_a(q_a)\\ q_a  \end{array} \right] \label{qtran1} \\
{\bf m}_{aa}^s (\bfq) &=& m_{aa}(q)-m_{au}(q) m_{uu}^{-1}(q) m_{au}^\top(q) \Big|_{q = \Phi^{-1}(\bfq) }, \\
{\bf m}_{uu} (\bfq) &=& m_{uu}(q) \Big|_{q = \Phi^{-1}(\bfq) }, \\
{\bf m}_{au} (\bfq) &=& m_{au}(q) \Big|_{q = \Phi^{-1}(\bfq) }.
\endequarr
\begin{proof} First notice that, under assumption {\bf A.2}, the coordinate transformation \eqref{qtran1} satisfies {\bf A.1} with 
\begequ
T(q)= \left[ \begin{array}{cc} I_s & -m_{uu}^{-1} m^{\top}_{au} \\ 0_{m\times s} & I_m \end{array} \right].
\endequ
Then, from Lemma \ref{lem1} we obtain that the dynamics can be written in the form \eqref{lagrcot} with 
\begequarr
\left[\begin{array}{c} \dot \bfq_1 \\ \dot \bfq_2 \end{array} \right] = \left[\begin{array}{cc} I_s & m_{uu}^{-1} m^{\top}_{au} \\ 0_{m\times s} & I_m  \end{array} \right] \left[\begin{array}{c} \dot q_u \\ \dot q_a \end{array} \right]
\endequarr
and Lagrangian
\begequ \label{Lcotu}
\mathcal L (\bfq,\dot \bfq) = \frac 12  \left[ \begin{array}{cc} \dot \bfq_1^\top & \dot \bfq_2^\top \end{array} \right] \left[ \begin{array}{cc} {\bf m}_{uu} & 0_{s\times m} \\ 0_{m\times s} & {\bf m}_{aa}^s \end{array} \right] \left[ \begin{array}{c} \dot \bfq_1 \\ \dot \bfq_2 \end{array} \right]  -  \mathcal V(\bfq).  
\endequ
The dynamics \eqref{ddq1_1}-\eqref{ddq2_1} follows, after some simple calculations, from the EL formula using the Lagrangian \eqref{Lcotu}.
\qed
\end{proof}
\begcor\label{cor1} The system \eqref{lagr} satisfying {\bf A.1-A.3} can be written as in the EL form as follows
\begequarr
&& m_{uu}(q_a) \ddot{\bfq}_1 +\nabla_{\bfq_1} \mathcal V(\bfq_1, q_a) =  G_u(q)u, \\
&& m^s_{aa} \ddot{q}_a + \left[ \nabla_{q_a} [m_{aa}^s(q_a) \dot q_a] - \frac 12 \nabla^\top_{q_a} [m_{aa}^s(q_a) \dot q_a]\right] \dot q_a + \nonumber \\
&& \nabla_{q_a} \mathcal V(\bfq_1, q_a) =  \bigg[ G_{a}(q)-G_{u}(q)m_{au}m^{-1}_{uu} \bigg] u,
\endequarr
with $m_{aa}^s (q_a) = m_{aa}(q_a)-m_{au}^\top(q_a) m_{uu}^{-1} m_{au}(q_u)$.

In addition, if assumption {\bf A.3-A.4} also holds, then the EL dynamics can be written as follows
\begequarr
&& m_{uu} \ddot{\bfq}_1  +   \nabla_{q_u} V_u \Big|_{q_u=\bfq_1 - \Phi_a(q_a)} = G_u(q)u , \label{plf11} \\
&& m_{aa}^s \ddot q_a + \left[ \nabla_{q_a} [m_{aa}^s \dot q_a] - \frac 12 \nabla_{q_a} [m_{aa}^s \dot q_a] \right] \dot q_a \nonumber \\
&& + \nabla_{q_a} V_a - m_{au} m_{uu} \nabla_{q_u} V_u \Big|_{q_u = \bfq_1 - \Phi_a(q_a)} = \nonumber\\
&&\bigg[ G_{a}(q)-G_{u}(q)m_{au}m^{-1}_{uu} \bigg] u. \label{plf12}
\endequarr
\endcor
\begin{proof}
The proof follows from Proposition \ref{prop1} and {\bf A.1-A.3} by setting in \eqref{ddq1_1}-\eqref{ddq2_1} the following conditions: $\bfq_1=q_u+\Phi_a(q_a)$, $\bfq_2 = q_a$, $m_{uu}$ is a constant matrix, and ${\bf m}_{aa}^s(\bfq)= {m}_{aa}^s(q_a)$. The second part follows from the fact that, under assumption {\bf A.4}, the potential function is $\mathcal V(\bfq)=V_a(q_a) + V_u(\bfq_1-\Phi_a(q_a))$.
\qed
\end{proof}
\begrem Notice that the system in the partial linear form \eqref{plf11}-\eqref{plf12} has been used to design a PID passivity-based controller in \cite{Salamat2021, Letizia2021}. In that work, an outer partial feedback linearization (PFL) control is used to obtain the desired form, which compromises the robustness of the closed loop. However, this PFL control can be avoided by using a generalized change of coordinates as shown in Corollary \ref{cor1}.
\endrem
The generalized coordinate transformation in Proposition \ref{prop1} is also useful (as it will be shown in the next section) for the holonomic swarm of particles.

\begin{center}
\begin{figure}[t]
\includegraphics[scale=0.72]{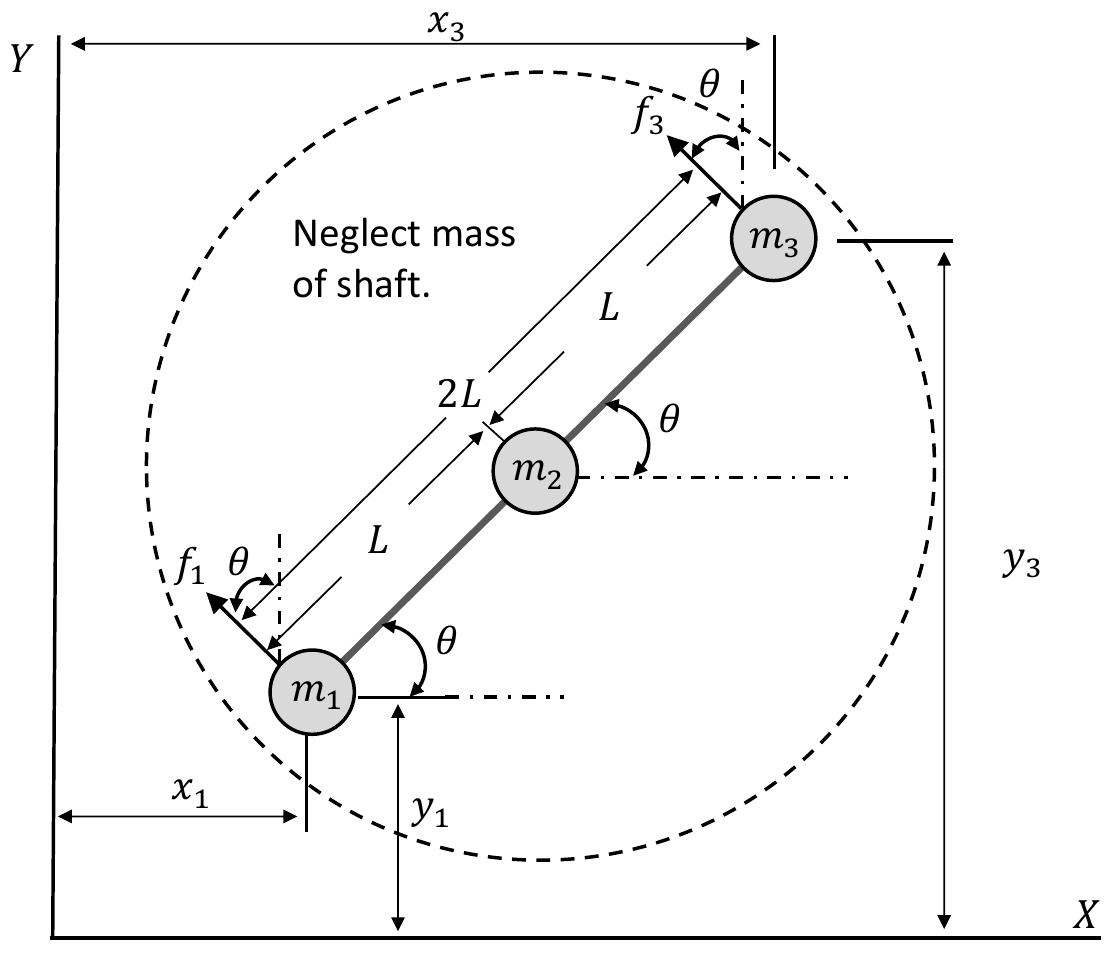}
\centering
\vspace{-0.1em}
\caption{Holonomic robot as a system of particles.}
\vspace{-1.0em}
\label{Robot}
\centering
\end{figure}
\end{center}
\section{Holonomic System}
\lab{sec3}

We consider a holonomic robot with masses $m_1$, $m_2$ and $m_3$, as shown in Fig.~\ref{Robot} that are rigidly fastened to the mass-less shaft and are free to move in the 2D plane. We now set up the equation of motion of the holonomic robot using convenient coordinates $q = [q_1, \hspace{0.2em} q_2, \hspace{0.2em} q_3]^\top = [x_1, \hspace{0.2em} y_1, \hspace{0.2em} \theta]^\top$. An external force $f_1$ is applied to $m_1$ in the direction of $-x_1$ and $y_1$ respectively, and $f_3$ to $m_3$ in the direction of $-x_3$ and $y_3$ respectively. To simplify the notation, we 
assume that all representative particle masses are the same (e.g., $m_i= m$ for $i=1,\hspace{0.2em} \dots,\hspace{0.2em} 3$). Applying Lagrange's equations, it immediately follows that 

\begequarr \label{Lagrange HR}
\mathcal L =\dfrac{1}{2} \dot{q}^{\top} \begin{bmatrix}
           3m &             0 & -3Lm\sin(\theta) \\[2pt]
             0 &           3m &  3Lm\cos(\theta) \\[2pt]
-3Lm\sin(\theta) & 3Lm\cos(\theta) &        5L^2m \\
\end{bmatrix}\dot{q}
\endequarr

where $(x_1,\hspace{0.2em} y_1)$ is positioned at the center of the first mass particle, $L$ is the distance between each mass, and $\theta$ is the inclination angle (see Fig.~\ref{Robot}). The equations of motion can be written in compact form as

\begequarr 
 \label{lagr HS}
M(q) \ddot{q} + C(q,\dot{q}) \dot{q} = G(q)  u,
\endequarr

where $M(q)$ is the generalized inertia matrix

\begequarr
\label{M HS}
M(q) =  \begin{bmatrix}
           3m &             0 & -3Lm\sin(\theta) \\[2pt]
             0 &           3m &  3Lm\cos(\theta) \\[2pt]
-3Lm\sin(\theta) & 3Lm\cos(\theta) &        5L^2m \\
\end{bmatrix}.
\endequarr
$C(q,\dot{q})$ is the  Coriolis matrix
\begequarr
C(q,\dot{q}) = \nonumber\\ 
\begin{bmatrix}
        0 &     0 &     -3L\dot{\theta}m\cos(\theta) \\[2pt]
        0 &     0 &      -3L\dot{\theta}m\sin(\theta)\\[2pt]
\dfrac{3L\dot{\theta}m\cos(\theta)}{2} & \dfrac{3L\dot{\theta}m\sin(\theta)}{2} & h_s\\
\end{bmatrix}\label{C HS}
\endequarr
with 
\begequarr \label{C HS r}
h_s = -\dfrac{3Lm(\dot{x}_1\cos(\theta) + \dot{y}_1\sin(\theta))}{2}. 
\endequarr

Therefore, the elements of the inertia matrix for the holonomic robot are given by

\begequarr
m_{uu} &=& 3m \nonumber\\
m_{au} &=& m_{au}^{\top} = \begin{bmatrix} 0 & -3Lm\sin(\theta) \end{bmatrix} \nonumber\\
m_{aa} &=& \begin{bmatrix} \label{elements of HR}         
3m              &    3Lm\cos(\theta) \\[2pt]
3Lm\cos(\theta) &              5L^2m \\
\end{bmatrix}.
\endequarr

The virtual work is given by

\begequarr 
 \label{WW}
\delta W = \begin{bmatrix}
-(f_1+f_3)\sin(\theta)\delta x_1 \\[2pt]
+ (f_1+f_3)\cos(\theta)\delta y_1 \\[2pt]
2Lf_3\delta \theta\\
\end{bmatrix}.
\endequarr

Without loss of generality, $G(q)$ can be written as

\begequ \label{G HR}
G(q) = \begin{bmatrix}
-sin(\theta) & 0 \\[2pt]
 cos(\theta) & 0 \\[2pt]
       0 & 1 \\
\end{bmatrix},
\endequ
with $u = [u_1, \hspace{0.2em} u_2]^{\top} = [f_1+f_3, \hspace{0.2em} 2Lf_3]^{\top} $.Therefore, $G_u(q) = \begin{bmatrix} -sin(\theta) & 0 \end{bmatrix} $ and $G_a(q) = \begin{bmatrix} \cos(\theta) & 0\\ 0 & 1\end{bmatrix}$. Note that $G_a(q)$ is an invertible $
2 \times 2$ matrix.

\subsection{Mechanical Properties of the Holonomic Robot}
The holonomic robot as defined by (\ref{lagr HS}) has several fundamental properties, which
can be used to facilitate the design of an automatic control mechanism.

\begite
\item[{\bf P.1}]  $M(q)$ is a positive definite matrix.

\item[{\bf P.2}] The inertia matrix depends only on the actuated variables $q_a$, i.e., $M(q) =M(q_a)$.

\item[{\bf P.3}] The sub-block matrix $m_{uu}$ of the inertia matrix is constant. 

\item[{\bf P.4}] From (\ref{M HS}) and (\ref{C HS}), and by using \ref{C HS r}, we get
\begequarr \label{skew-symmetrie}
&&  \hspace{-6em}\dot{M}(q) - 2C(q,\dot{q}) = \nonumber \\
&& \hspace{-6em}\begin{bmatrix}                     0&                      0& \dfrac{9L\dot{\theta}m\cos(\theta)}{2} \\[2pt]
                     0&                      0& \dfrac{9L\dot{\theta}m\sin(\theta)}{2} \\[2pt]
-\dfrac{9L\dot{\theta}m\cos(\theta)}{2}& -\dfrac{9L\dot{\theta}m\sin(\theta)}{2}&                     0 \\
\end{bmatrix},
\endequarr
which is a skew-symmetric matrix.
\endite

The system has three degrees of freedom and only two actuators, hence, we have an underactuated mechanical system. We have nonlinearities because the generalized inertia matrix is off-diagonal and the input matrix is highly coupled. Due to the lack of more actuators, this system cannot be fully linearized using exact feedback linearization. 
However, it is still possible to apply PFL to the system, such that the translational dynamics $q_1 = x_1$ and $q_2 = y_1$ become a double integrator. As already mentioned before, PFL compromizes the robustness of the closed-loop. However, the PFL can be avoided by using the proposed transformation, as shown in Corollary~\ref{cor1}. Given the properties {\bf P.1}—{\bf P.4}, we apply the generalized coordinate transformations based on Proposition~\ref{prop1} to decouple the system.

\begin{figure*}[b]
\begequarr \nonumber
\begin{bmatrix} \dot{x}_1 \\[6pt] 
\ddot{x}_1 \\[6pt]
\dot{y}_1 \\[6pt]
\ddot{y}_1 \\[6pt]
\dot{\theta} \\[6pt]
\ddot{\theta} \\
\end{bmatrix}=\begin{bmatrix} 0 & 1 & 0 & 0 & 0 & 0 \\[6pt]
0& 0& 0& 0& 3Lm\sin(\theta)\dot{\theta}(-\dot{\theta} +1)+5L& 3Lm\cos(\theta)(2\dot{\theta} - 1)\\[6pt]
0 & 0 & 0 & 1 & 0 & 0\\[6pt]
0& 0& 0& 0& L\dot{\theta}^2\cos(\theta)+3sin(\theta) &   2L\dot{\theta}\sin(\theta)\\[6pt]
0 & 0 & 0 & 0 & 0 & 1\\[6pt]
0 & 0 & 0 & 0 & 0& 0\\
\end{bmatrix} \begin{bmatrix} x_1 \\[6pt] 
\dot{x}_1 \\[6pt]
y_1 \\[6pt]
\dot{y}_1 \\[6pt]
\theta \\[6pt]
\dot{\theta} \\
\end{bmatrix}
+ \begin{bmatrix}   0&  0\\[2pt]
         -Lsin(\theta)&  0\\[2pt]
                        0 &                0\\[2pt]
        \dfrac{5L\cos(\theta)}{6m}& -\dfrac{\cos(\theta)}{2m}\\[2pt]
                        0&                0\\[2pt]
               -\dfrac{1}{2m}&      \dfrac{1}{2m}\\
               \end{bmatrix}\begin{bmatrix} u_1\\[4pt]
u_2\\
\end{bmatrix}
\endequarr 
\end{figure*}

\begin{proposition} \label{prop2}
Considering the holonomic robot in (\ref{lagr HS}), the dynamical system model can be rewritten as

\begequarr\label{DS1}
&&\ddot{x}_1 = f_{x_1}(\theta, \dot{\theta}) - u_1L\sin(\theta), \\ \label{DS2}
&&\ddot{y}_1 = f_{y_1}(\theta, \dot{\theta}) + \dfrac{5\cos(\theta)}{6m}u_1 - \dfrac{\cos(\theta)}{2m}u_2,\\\label{DS3}
&&\ddot{\theta} =  -\dfrac{1}{2m}u_1 + \dfrac{1}{2m}u_2 \;. 
\endequarr
\end{proposition}

\begin{proof}
By applying Proposition \ref{prop1} the result follows.
\qed
\end{proof}
In the next section, we address questions related to the automatic control of a particle swarm that minimizes energy by applying the transformed underactuated model in (\ref{DS1})-(\ref{DS3}). 
We prove that the mean of configuration variables is controllable and provide conditions under which the variance is also controllable.

\section{Automatic Control for a Holonomic Swarm}
In this section, we present an automatic controller for a swarm of particles that minimizes the energy. 
We show that it only relies on the first two moments of the swarm configuration variables, i.e., the position and the orientation angle distribution. The main objective of our automatic control approach is to act on forces optimally so that particles can reach the desired target position $q^{\ast}=[x^{\ast}, \hspace{0.2em} y^{\ast}, \hspace{0.2em} \theta^{\ast}]^\top$ with the stable Euler angle ($\lim\limits_{t \to \infty} \theta =0$ ).

\subsection{Swarm Dynamical System Model}
By defining $x = (x_1,\hspace{0.2em} \dot{x}_1, \hspace{0.2em} y_1, \hspace{0.2em} \dot{y}_1,\hspace{0.2em} \theta, \hspace{0.2em} \dot{\theta}  )$, the dynamics of (\ref{DS1})-(\ref{DS3}) can be written as

\begequarr
\dot{x} = A(x)x + B(x)u.  \label{SDC}
\endequarr
The elements of $A (x)$ and $B(x)$ are given at the bottom of this page. 
The system is nonlinear, since matrices $A(x)$ and $B(x)$ both depend on the current state variables. Firstly, we analyze the number of controllable states as given by the following definition. 

\begin{definition}\label{Def1}
The states in (\ref{SDC}) are controllable if the pair $\{A(x),\hspace{0.2em} B(x)\}$ is point-wise controllable. This can be observed by the rank of the controllability matrix
\begequarr
\mathfrak{C} = \begin{bmatrix} B(x)\\[2pt]
A(x)B(x)\\[2pt]
\vdots\\[2pt]
A^5(x)B(x)\\
\end{bmatrix}.
\endequarr
\end{definition}
The consequence of Definition \ref{Def1} is that the $\mathfrak{C}$ matrix for the system in (\ref{SDC}) has the full rank (i.e., $\rank(\mathfrak{C})=6$). Therefore, all states are controllable. Previous work has shown that the mean and variance of many particles for simple fully actuated particles are controllable~\cite{Shahrokhi2015,Shahrokhi2018}. 
Next, we show how we can 
stabilize the nonlinear underactuated particles by a global state-feedback controller designed via state-dependent Riccati equation (SDRE) control~\cite{Cloutier1997}. 
Motivated by (\ref{SDC}) and defining the mean states $\bar{x}$ that represent the mean states of $N$ particles, we can write the dynamical system model of the swarm as

\begequarr
\dot{\bar{x}} = A(\bar{x})\bar{x} + B(\bar{x})u \;.  \label{Mean SDC}
\endequarr

Interestingly, analyzing the controllability of the swarm dynamics results in the same form as in (\ref{SDC}), hence, the mean states are controllable.

\subsection{Control Law}
Our objective is to find minimum energy inputs that steer the swarm to a given target state defined on $t \in [t_0,  \hspace{0.2em} t_f] = [0, \hspace{0.2em}  \infty]$. To do so, consider now the following cost functional

\begequarr 
\mathcal{J} = \dfrac{1}{2} \int^{\infty}_{0} \bar{x}^{\top}\mathcal{Q}(\bar{x})\bar{x} + u^{\top}\mathcal{R}(\bar{x})u \;, \label{CF}
\endequarr
with respect to the state $\bar{x}$ and control input $u$ subject to the nonlinear dynamical system model constraint
\begequarr
\dot{\bar{x}} = A(\bar{x})\bar{x} + B(\bar{x})u,  \nonumber
\endequarr
where $\mathcal{Q}(\bar{x}) \geq 0$ penalizes the state, and $\mathcal{R}(\bar{x}) >0$  penalizes the control effort for all $\bar{x}$.  
We aim for a nonlinear state-feedback controller $u$ that stabilize solutions of problem (\ref{Mean SDC})-(\ref{CF}). 
\begin{remark}
Cloutier~\cite{Cloutier1997} obtains the nonlinear feedback controller via SDRE. Our interest is to provide an alternative interpretation of solving the problem (\ref{Mean SDC})-(\ref{CF}) via Pontryagin’s minimum principle~\cite{kirk1970optimal}.
\end{remark}
From (\ref{Mean SDC}) and (\ref{CF}), the Hamiltonian can be written as
\begequarr \nonumber
H &=& \dfrac{1}{2} \bigg[\bar{x}^{\top}\mathcal{Q}(\bar{x})\bar{x} + u^{\top}\mathcal{R}(\bar{x})u \bigg] \\ 
&&+ p^{\top}(t) \bigg[ A(\bar{x})\bar{x} + B(\bar{x})u   \bigg], \label{Hamiltonian}
\endequarr
where $p(t)$ is the adjoint vector. The necessary condition is derived by differentiating (\ref{Hamiltonian}) with respect to $u$ which yields
\begequarr \label{NC1}
&&\nabla_{u} H = \mathcal{R}(\bar{x})u + B^{\top}(\bar{x})p = 0 \;.
\endequarr
We obtain the nonlinear feedback controller 
\begequarr \label{NFC}
u = -\mathcal{R}^{-1}(\bar{x})B^{\top}(\bar{x})p \;.
\endequarr
Now, we define $p \triangleq P(x)x$, where the matrix $P(x)$ can be obtained by solving the algebraic Riccati equation
\begequarr \nonumber
&&A^{\top}(\bar{x})P+PA(\bar{x})-PB(\bar{x})\mathcal{R}^{-1}(\bar{x})B^{\top}(\bar{x})P\\ \label{Riccati}
&&+\mathcal{Q}(\bar{x})=0.
\endequarr

By that we fulfill the second optimality condition
\begequarr \label{NC2}
\dot{p} = -\nabla_{\bar{x}} H(\bar{x}, t, p) \;.
\endequarr
Therefore, as long as the two conditions in (\ref{NC1}) and (\ref{NC2}) hold, it is always possible to construct a nonlinear feedback controller that solves the problem (\ref{Mean SDC})-(\ref{CF}). The closed-loop solution for this a feedback controller is at least a local optimum and possibly the global optimum.

\subsection{Stability Analysis} 
\begin{theorem}\label{Th1}
Consider the dynamical system model (\ref{Mean SDC}), with the feedback controller (\ref{NFC}). Assume in addition that for a constant input weighting matrix $\mathcal{R} > 0$, the state weighting matrix $Q(x) >0$ can be chosen, such that $\dot{P}(x) <0$ for all $x$,  where $P(\bar{x})$ is the solution of (\ref{Riccati}). Then the zero equilibrium of the closed-loop system is semi-globally stable.
\end{theorem}

\begin{proof}
Consider the Lyapunov function candidate
\begequarr \label{Ly Func}
L(\bar{x}) = \bar{x}^{\top}P(\bar{x})\bar{x},
\endequarr

the time derivative of which, along the trajectories of the closed-loop dynamical system, is such that
\begequarr \nonumber
&&\dot{L}(\bar{x}) = \dot{\bar{x}}^{\top}P(\bar{x})\bar{x} + \bar{x}^{\top}P(\bar{x})\dot{\bar{x}} + \bar{x}^{\top}\dot{P}(\bar{x})\bar{x} \\
&&= \bar{x}^{\top} \bigg[ \dot{P}(\bar{x}) - \mathcal{Q}(\bar{x}) \nonumber
\\
&&-P(\bar{x})B(\bar{x})\mathcal{R}^{-1} B^{\top}(\bar{x})P(\bar{x})   \bigg]\bar{x}, \label{D Ly Func}
\endequarr

where $P(\bar{x})B(\bar{x})\mathcal{R}^{-1} B^{\top}(\bar{x})P(\bar{x}) >0$. In addition, based on the assumed selection of $Q(\bar{x})$, yields $\dot{P}(\bar{x}) <0$ and $\dot{L}(x)<0$, hence our claim. 
\qed
\end{proof}

\begin{algorithm}[b]
\SetAlgoLined
\textbf{begin procedure}\\
Step 1: Partition the generalised coordinates and velocity $q=\col(q_u,q_a)$, ${\dot q}=  \col ({\dot q}_u, {\dot q}_a)$.\\
Step 2: Construct the invertible mapping\\
\begequ \nonumber
\dot \Phi(q) = T(q) \dot{q},
\endequ
\\
with \\
\begequ \nonumber
T(q)= \left[ \begin{array}{cc} I_s & -m_{uu}^{-1} m^{\top}_{au} \\ 0_{m\times s} & I_m \end{array} \right].
\endequ

Step 3: Apply Proposition \ref{prop1}. 

\textbf{end procedure}
\caption{Generalised Coordinate Transformation for Underactuated Particle (\ref{lagr HS}).}
\label{Al1}
\end{algorithm}

\begin{algorithm}[b]
\SetAlgoLined
\caption{Hysteresis-based Mean and Variance Automatic Control}
\label{Al2}
 \begin{algorithmic}
    \Require Knowledge of underactuated particle swarm mean $\bar{x}$, variance $[\sigma^{2}_{x_{1}} \hspace{0.2em}   \sigma^{2}_{y_{1}} ]^\top$, the boundary of the search space $\{ x_{min} \hspace{0.2em}x_{max} \hspace{0.2em}y_{min} \hspace{0.2em}y_{max}\}$, and the desired mean state $q^{\ast}=[x^{\ast}, \hspace{0.2em} y^{\ast}, \hspace{0.2em} \theta^{\ast}]^\top$.\\
    
     $x_{goal} \leftarrow x^{\ast}$ , $y_{goal} \leftarrow y^{\ast}$\;
    \Loop\\
    \hspace{2.2em} \eIf{$\sigma^{2}_{x_{1}} > \sigma_{max}$}{
     $x_{goal} \leftarrow  x_{min}$\;}
     {
     $x_{goal} \leftarrow  x^{\ast}$
     }\;
    \eIf{$\sigma^{2}_{y_{1}} > \sigma_{max}$}
    {
     $y_{goal} \leftarrow  y_{min}$\;}
     {
     $y_{goal} \leftarrow  y^{\ast}$
    }\\
    Apply the automatic control law (\ref{NFC}) to regulate the underactuated swarm to the desired state $q^{\ast}=[x^{\ast}, \hspace{0.2em} y^{\ast}, \hspace{0.2em} \theta^{\ast}]^\top$
   \EndLoop
\end{algorithmic}
\end{algorithm}

\subsection{Controlling Mean and Variance}
The variances $\sigma^{2}_{x_{1}}$ and $\sigma^{2}_{y_{1}}$ of $N$  underactuated particle’s position is
\begequarr 
\bar{x}_{1} = \dfrac{1}{N}\sum^{N}_{i=1} x_{1i} , \hspace{1.5em} \sigma^{2}_{x_{1}} = \dfrac{1}{N}\sum^{N}_{i=1} (x_{1i}- \bar{x}_{1})^{2}, \\ \label{Varx}
\bar{y}_{1} = \dfrac{1}{N}\sum^{N}_{i=1} y_{1i} , \hspace{1.5em}\sigma^{2}_{y_{1}} = \dfrac{1}{N}\sum^{N}_{i=1} (y_{1i}- \bar{y}_{1})^{2}. \label{Vary}
\endequarr

The objective now is to control both, the mean and variance, effectively to ensure approaching a target position with minimum variance. Therefore, the selected strategy is the hysteresis-based approach following~\cite{Shahrokhi2015, Kloetzer2007}. The idea is that the automatic controller regulates the mean states of $N$ underactuated particles with radius $r$ but switches to minimizing variance if the variance exceeds the threshold $\sigma_{max} = 2.5r + \sigma^2_{optimal}(n,r)$ and until $\sigma_{min}=15r + + \sigma^2_{optimal}(n,r)$ is reached \cite{Shahrokhi2015}. The idea of using such values comes from Graham and Sloane \cite{Graham90}. They proved that the minimum variance to collect $N$ 2D circles with radius $r$ is $0.55Nr^2$.

Our proposed methodology in total consists of (1)~applying the generalized coordinate transformation shown in Algorithm~\ref{Al1} and (2)~proposing and analyzing the control mechanism to regulate the mean and variance of the swarm of underactuated particles shown in Algorithm \ref{Al2}.

\subsection{Fully-actuated vs Underactuated Particle Swarm}

We now consider a small swarm of $N=4$ particles to showcase the performance of the proposed control law and highlight the advantage of the underactuated particle swarm over the fully-actuated swarm~\cite{Shahrokhi2015}. The sampling time is set to $0.01$~s and the physical parameters are given in Table~\ref{tab:1}. The control gain matrices $\mathcal{Q}(\bar{x})$ and $\mathcal{R}$ are based on the assumptions of Theorem~\ref{Th1} and we get

\begequarr  \nonumber
&&\mathcal{Q}(\bar{x})= \bigg(1+0.01(\bar{x}_{5}-0)^2+0.01(\bar{x}_{6}-0)^2\bigg)\\ \nonumber
&&
0.001\,\diag(260\hspace{0.2em}, 1 \hspace{0.2em},260 \hspace{0.2em},1 \hspace{0.2em},160 \hspace{0.2em},100) ,\hspace{0.2em} \mathcal{R} = \begin{bmatrix} 50 & 0 \\[2pt]
0 & 50
\end{bmatrix}.
\endequarr
\begin{figure*}[t]
\begin{center}
\includegraphics[scale=0.4]{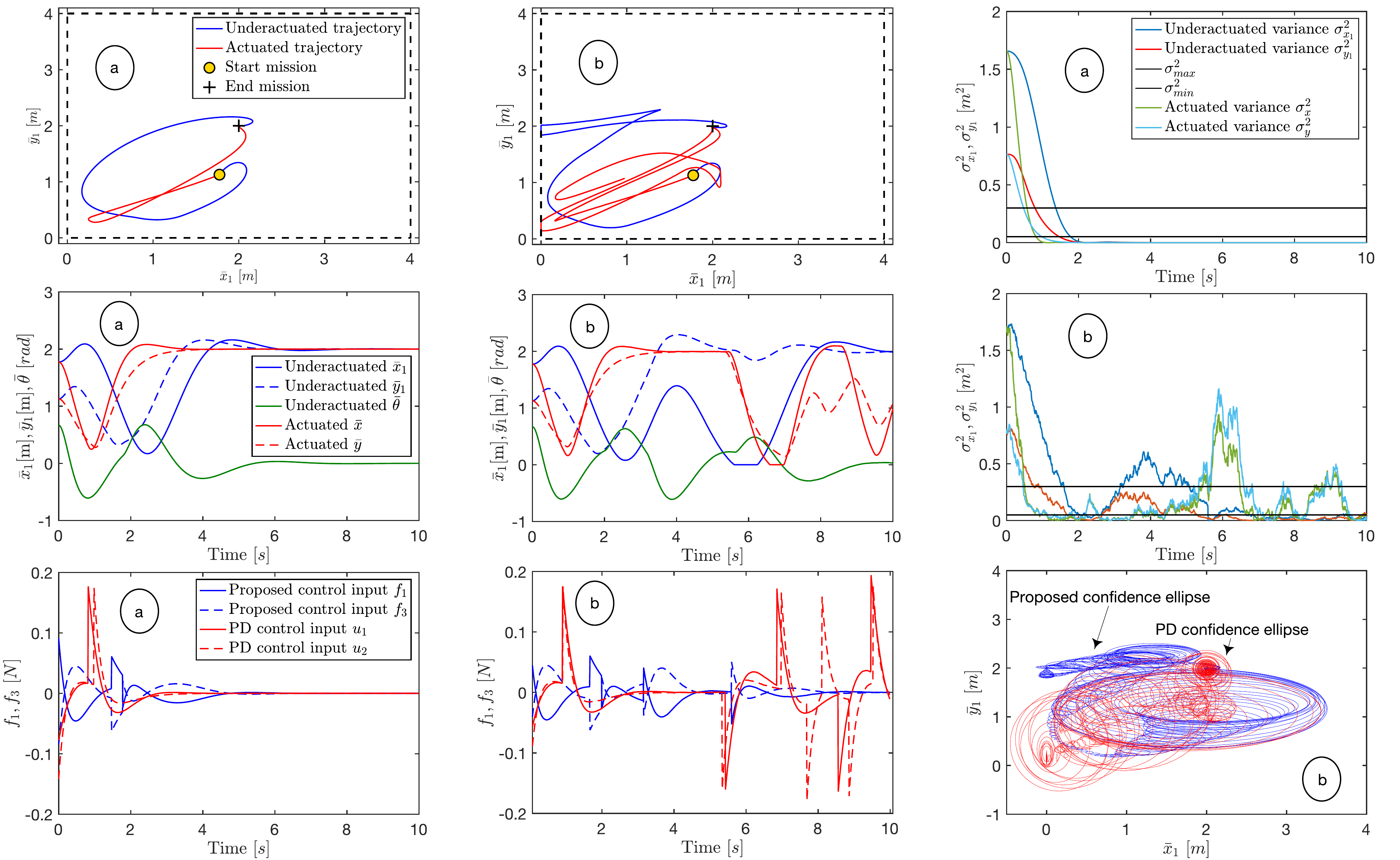} 
\centering
\vspace{-0.5em}
\caption{ Simulation of control laws. (a): Automatic control of the mean and variance  with 4 particles in the search space (black dashed line). Underactuated particle swarm under control algorithm \ref{Al1} and \ref{Al2} over the fully-actuated \cite{Shahrokhi2015}. (b): In the simulation, increased Brownian noise results in a more agile increase of variance.}   

\vspace{-1.5em}
\label{CL}
\end{center}
\end{figure*}

\begin{table}[h]
\centering
\caption{Parameters of the simulated underactuated swarm}
\label{tab:1}
\begin{tabular}{|c|c|c|c|}
\hline
\textbf{Parameters}      & \textbf{Symbol} &  \textbf{Value} & \textbf{Unit}\\ \hline
mass            & $m$    & $0.01$ & kg \\ \hline
shaft & $L$     & $0.02$      & m \\ \hline
\end{tabular}
\end{table}

\begin{table}
  \begin{center}
    \begin{tabular}{r|l l}
      \textbf{Topic} & \textbf{Fully-actuated~\cite{Shahrokhi2015}} & \textbf{Underactuated} \\ \hline
      Language & JavaScript & \href{https://www.typescriptlang.org/}{TypeScript} \\  
      Physics Framework & \href{https://box2d.org/}{Box2D}~\cite{box2D} & \href{https://brm.io/matter-js/}{matter.js}~\cite{matterjs} \\
      Rendering & HTML, CSS, JS & \href{https://reactjs.org/}{react} \& \href{https://p5js.org/}{p5.js}
    \end{tabular}
    \caption{Technologies used in the robot simulations.}
    \label{tbl:diffTechs}
  \end{center}
\end{table}

\begin{figure}[h] 
    \centering
  \subfloat[frame 1]{%
       \includegraphics[width=0.46\linewidth]{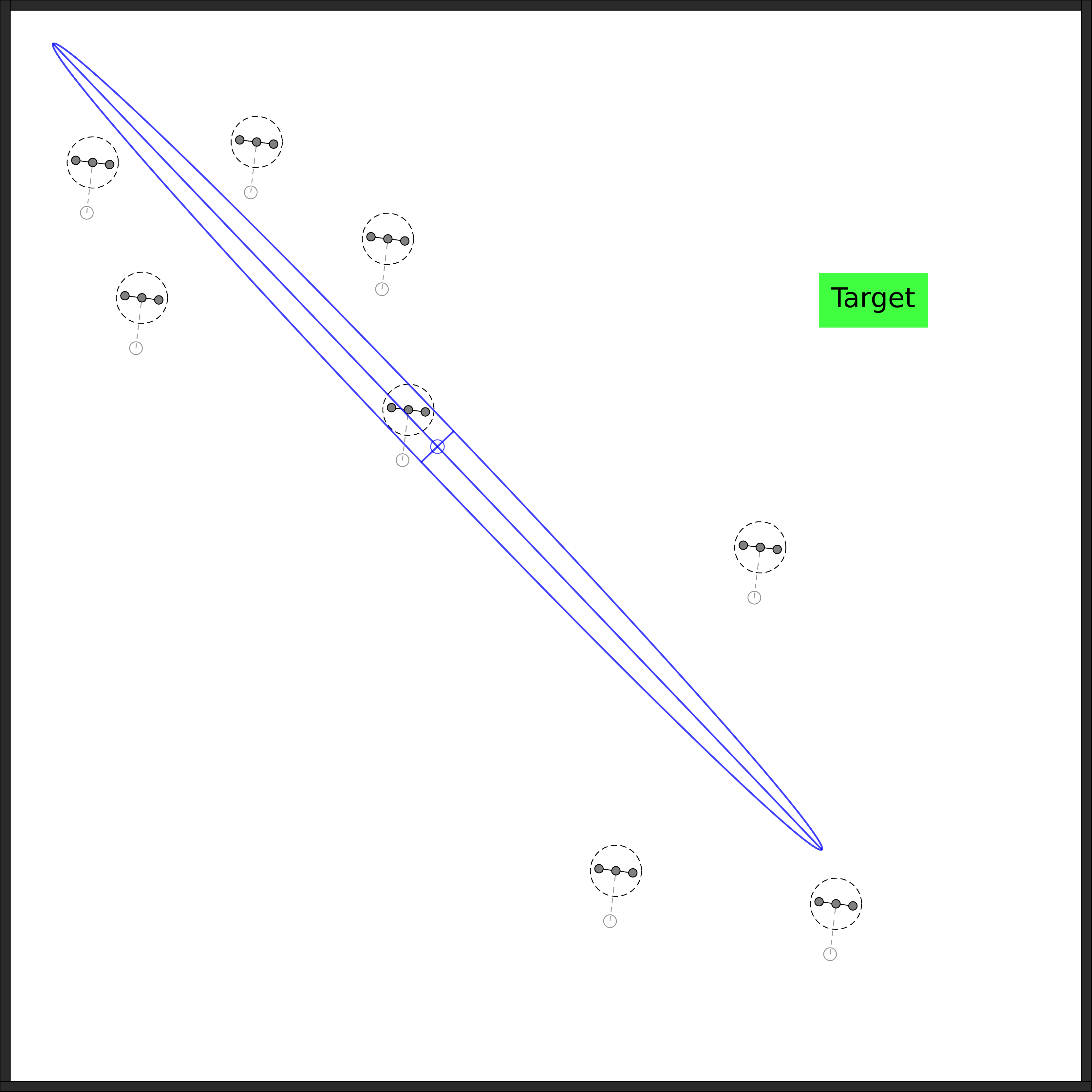}}
    \hfill
  \subfloat[frame 11]{%
        \includegraphics[width=0.46\linewidth]{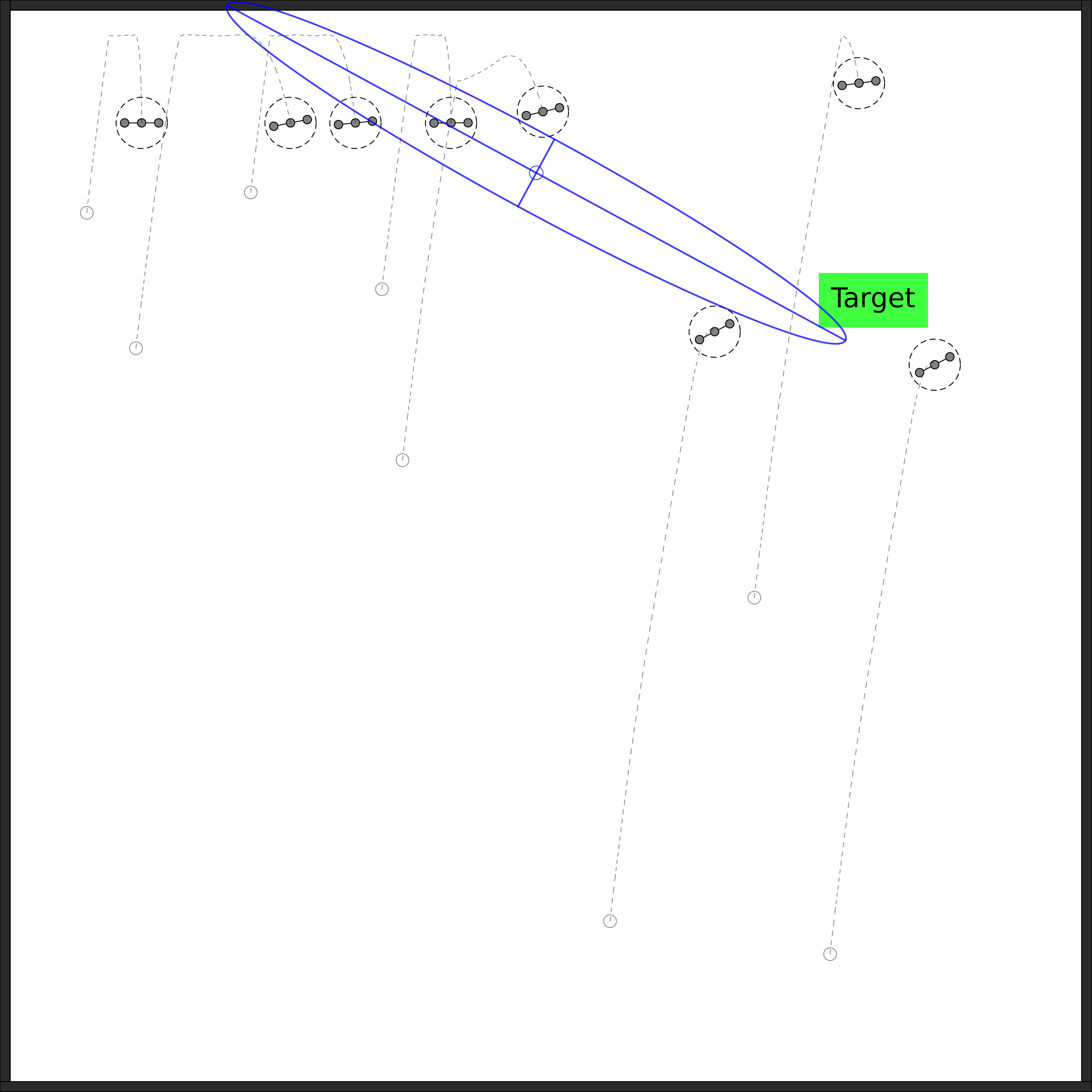}}
    \\
  \subfloat[frame 35]{%
        \includegraphics[width=0.46\linewidth]{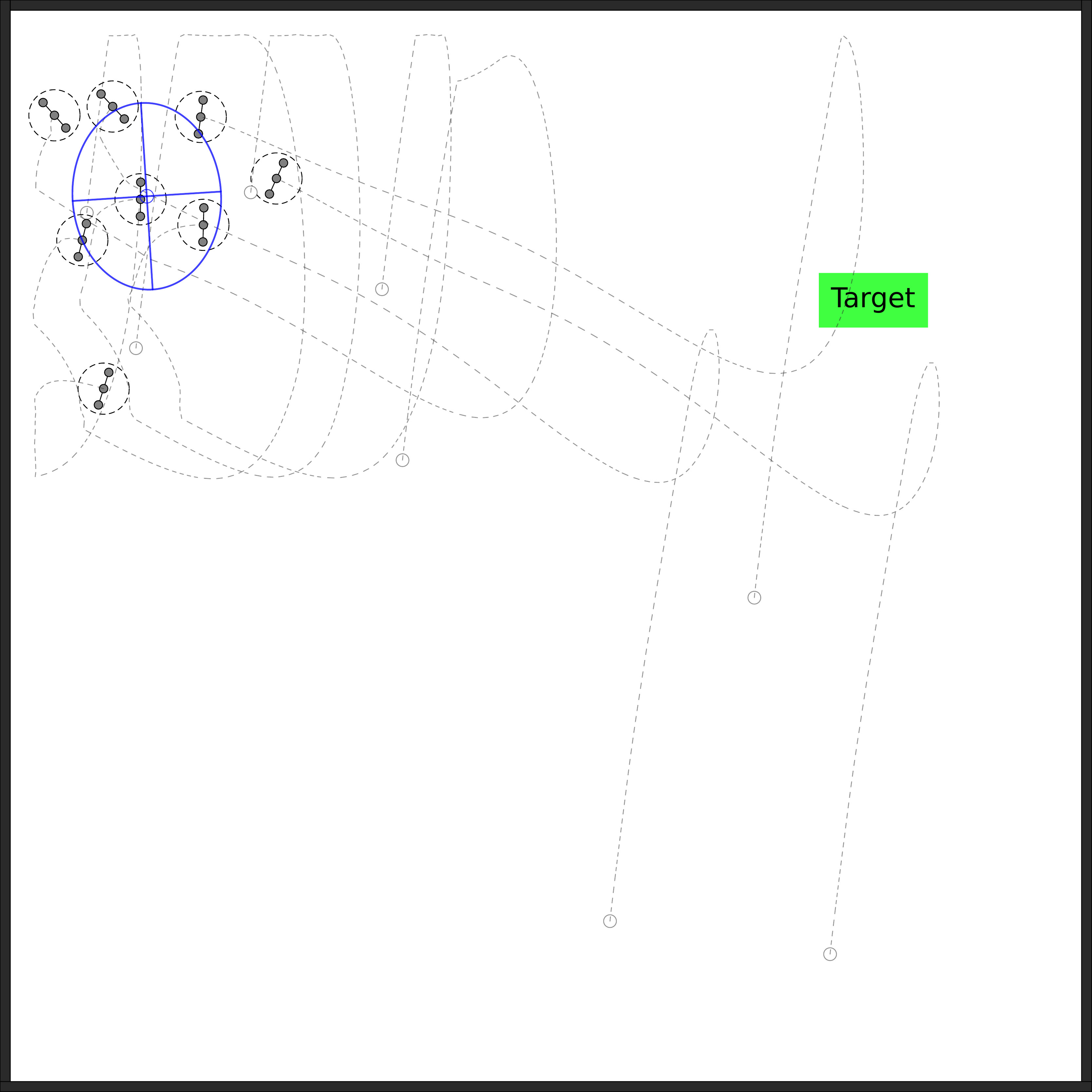}}
    \hfill
  \subfloat[frame 49]{%
        \includegraphics[width=0.46\linewidth]{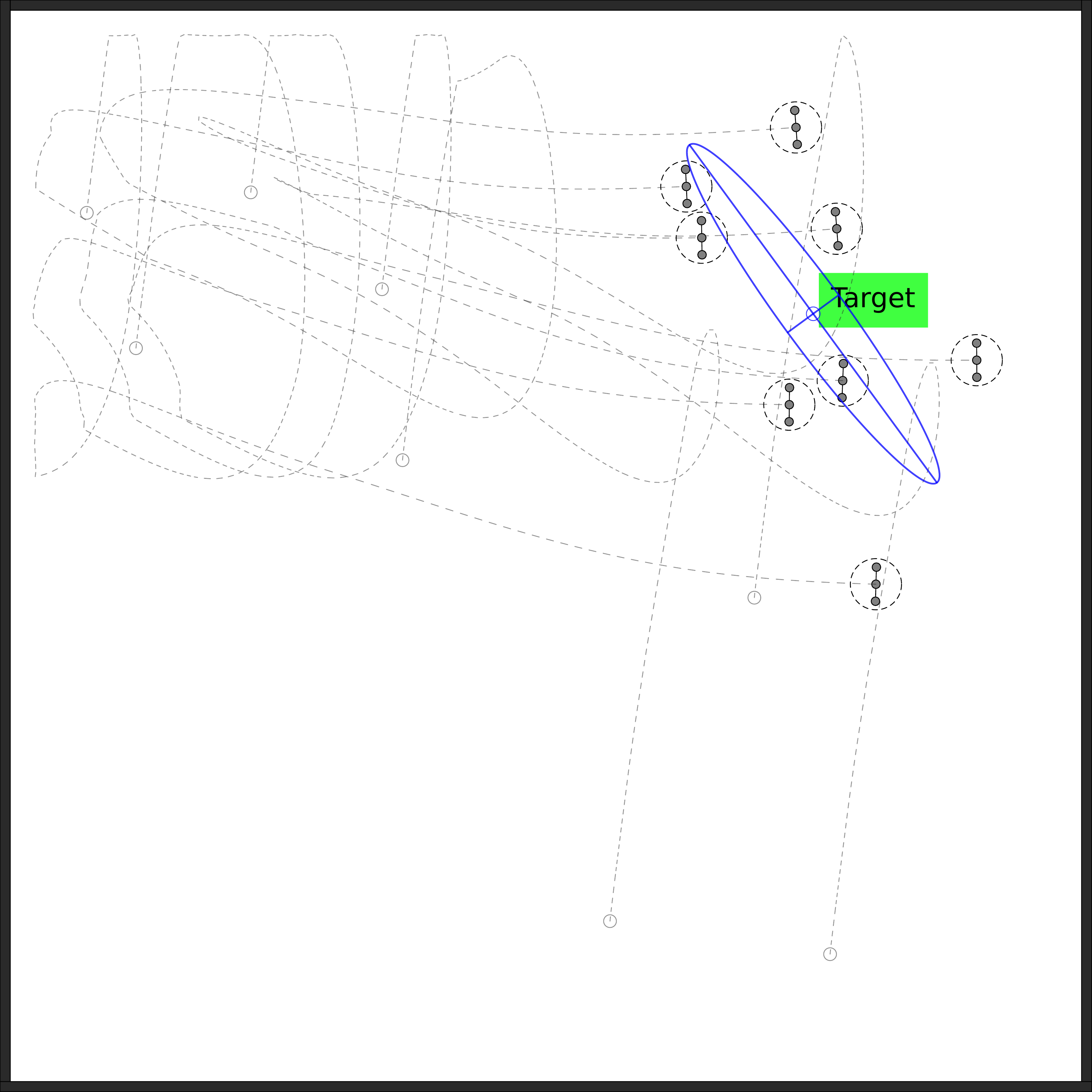}}
  \caption{Different stages of $N=8$  underactuated robots using our nonlinear controller (eq.~\ref{NFC}). (a)~Initial condition. (b)~Hysteresis-based control following Algorithm~\ref{Al2}. (c)~Minimizing mean and variance of robot positions utilizing the arena boundary. (d)~Regulating mean state and stabilizing the Euler angle~$\theta$.}
  \label{fig:robot simulation}
\end{figure}


We compare to the approach of Shahrokhi et al.~\cite{Shahrokhi2015}. 
Their control gains for the PD controller are $K_{p_{x}} = 0.04$, $K_{p_{y}} = 0.03$, $K_{d_{x}} = 0.03$, and $K_{p_{y}} = 0.04$.

Fig. \ref{CL} compares our approach and the PD controller~\cite{Shahrokhi2015} for the obtained trajectories of mean, variance, and the control inputs. 
Even though the settling times seem satisfactory for both approaches, 
the trajectory and the control inputs 
allow to discriminate the two approaches. 
The control inputs obtained through our approach are significantly smaller resulting in less energy consumption. Also note that there 
are no sudden peaks in the control inputs. 
The fully-actuated approach consumes $\SI{1.0347}{\J}$ of energy compared to the under-actuated one that consumes~$\SI{0.4639}{\J}$. This is an energy reduction of approximately $57 \%$.

Both approaches minimize mean and variance. However, in the underactuated case, we stabilize the mean Euler angle~$\bar{\theta}$ with only two global control inputs. Hence, we reasonably balance the tradeoff between control complexity and system performance. 

\section{Multi-Robot Simulations}
We also show the result of a simple robot simulation for a swarm of $N=8$ robots to visualize our results in an accessible way. The software frameworks used to implement the simulations by Shahrokhi et al.~\cite{Shahrokhi2015} and ours are given in Table~\ref{tbl:diffTechs}. Time is discretized and the control signal is scaled by~$\delta t = 0.01$. The underactuated robots and arena boundaries are simulated as physical entities. 
Each underactuated robot has a random initial pose and the swarm's mean position has a randomly generated target pose.
The nonlinear controller described in Algorithm~\ref{Al2} steers the robots 
from a starting position to a target position (equilibrium point of the swarm) with a stable Euler angle. 
Fig.~\ref{fig:robot simulation} shows 
four screenshots during a representative simulation run. 
This result shows how the properties of the underactuated robot system (e.g., torque and inertia) 
are exploited to regulate the mean, to minimize the variance, and to steer the swarm to the target in the right pose.

\section{Conclusion}
We have proposed a centralized automatic stochastic control of large-scale robot systems for underactuated robots based on generalized change of coordinates. 
We transform underactuated robot systems to the partial form 
that can be use for control design.
At the cost of centrally tracking all robots, we gain the benefit of an optimal energy-efficient control in a task of minimizing positional variance and moving the robot system's mean to a goal position.
The requirement of having to track all robots is unlikely to scale arbitrarily. A~future extension of our method could hence be to track only the boundary of a `cloud of robots' and their center of gravity. Possibly also the particle density could be measured instead of each individual robot.
There is no immediate way of transferring our method to a
decentralized approach, hence making it complementary to behavior-based approaches from swarm robotics that show increased robustness without a central element. However, centralized and decentralized approaches and their pros and cons are complementary to each other, which needs to be carefully considered by the designer for a given use case.
In future work, we plan to test and study our approach on real robots with different physical characteristics, such as the Kilobot and other robots with bigger masses. Also an extension of the method to a manipulation scenario~\cite{Shahrokhi2018} seems particularly relevant.

%

\bibliographystyle{IEEEtran}
\bibliography{bib}

\end{document}